\documentclass{article}

\usepackage[dvipsnames]{xcolor}
\usepackage[utf8]{inputenc} %
\usepackage[T1]{fontenc}    %
\usepackage{hyperref}       %
\usepackage{url}            %
\usepackage{booktabs}       %
\usepackage{amsfonts}       %
\usepackage{nicefrac}       %
\usepackage{microtype}      %
\usepackage{xcolor}         %
\usepackage{algorithmic}
\usepackage{algorithm}
\usepackage{amsmath, amsfonts, amssymb, amsthm}
\usepackage{bbm}
\usepackage{bbold}
\usepackage{graphicx}

\newtheorem{theorem}{Theorem}
\newtheorem{lemma}{Lemma}
\newtheorem{proposition}{Proposition}
\newtheorem{definition}{Definition}
\newtheorem{assumption}{Assumption}

\title{A Challenge in Reweighting Data \\ with Bilevel Optimization}

\author{Anastasia Ivanova$^1$\thanks{work done while at Apple} \and Pierre Ablin$^2$}
\date{
	$^1$ Univ. Grenoble Alpes \\
	$^2$ Apple
}

\begin{document}

\maketitle
\begin{abstract}
In many scenarios, one uses a large training set to train a model with the goal of performing well on a smaller testing set with a different distribution.
Learning a weight for each data point of the training set is an appealing solution, as it ideally allows one to automatically learn the importance of each training point for generalization on the testing set.
This task is usually formalized as a bilevel optimization problem.
Classical bilevel solvers are based on a warm-start strategy where both the parameters of the models and the data weights are learned at the same time.
We show that this joint dynamic may lead to sub-optimal solutions, for which the final data weights are very sparse.
This finding illustrates the difficulty of data reweighting and offers a clue as to why this method is rarely used in practice.
\end{abstract}

\section{Introduction}
In many practical learning scenarios, there is a discrepancy between the training and testing distribution.
For instance, when training large language models, we may have access to a training set that contains many low-quality data points from different sources and want to train a model on this dataset to perform well on a testing set that contains a few high-quality points~\cite{devlin2019bert,bommasani2021opportunities,mahajan2018exploring}.
An appealing way to solve this problem is \emph{data reweighting}~\cite{ren2018reweight,shu19metaweight,wang2020optimizing}, where one attributes one weight to each data point in the training set.
The weight of a training sample should reflect how much this sample resembles the testing set and helps the model perform well on it.
\autoref{fig:basic_concept} illustrates the general principle.

Learning the optimal weights can be cast as a bilevel optimization problem~\cite{franceschi2017forward}, where the optimal weights are such that training the model with these weights leads to the smallest test loss possible.
The weights are usually constrained to sum to one, leading to an optimization problem on the simplex, which is usually solved with mirror descent~\cite{nemirovskij1983problem}.
Despite its promise of automatically learning the importance of data points, data reweighting is still seldom used in practice.
In this paper, we try to provide a possible explanation for this lack of adoption, showing, in short, that the underlying optimization problem is hard.

In a large-scale setting where fitting the model once is expensive, it is prohibitively costly to iteratively update the weights with a model fit at each iteration~\cite{pedregosa2016hyperparameter}.
Hence, practitioners often resort to \emph{warm-started} bilevel optimization, where the parameters of the model and the weights evolve simultaneously~\cite{luketina2016scalable,ji2021bilevel,dagreou2022framework}.
\begin{figure}[t]
    \includegraphics[width=.99\columnwidth]{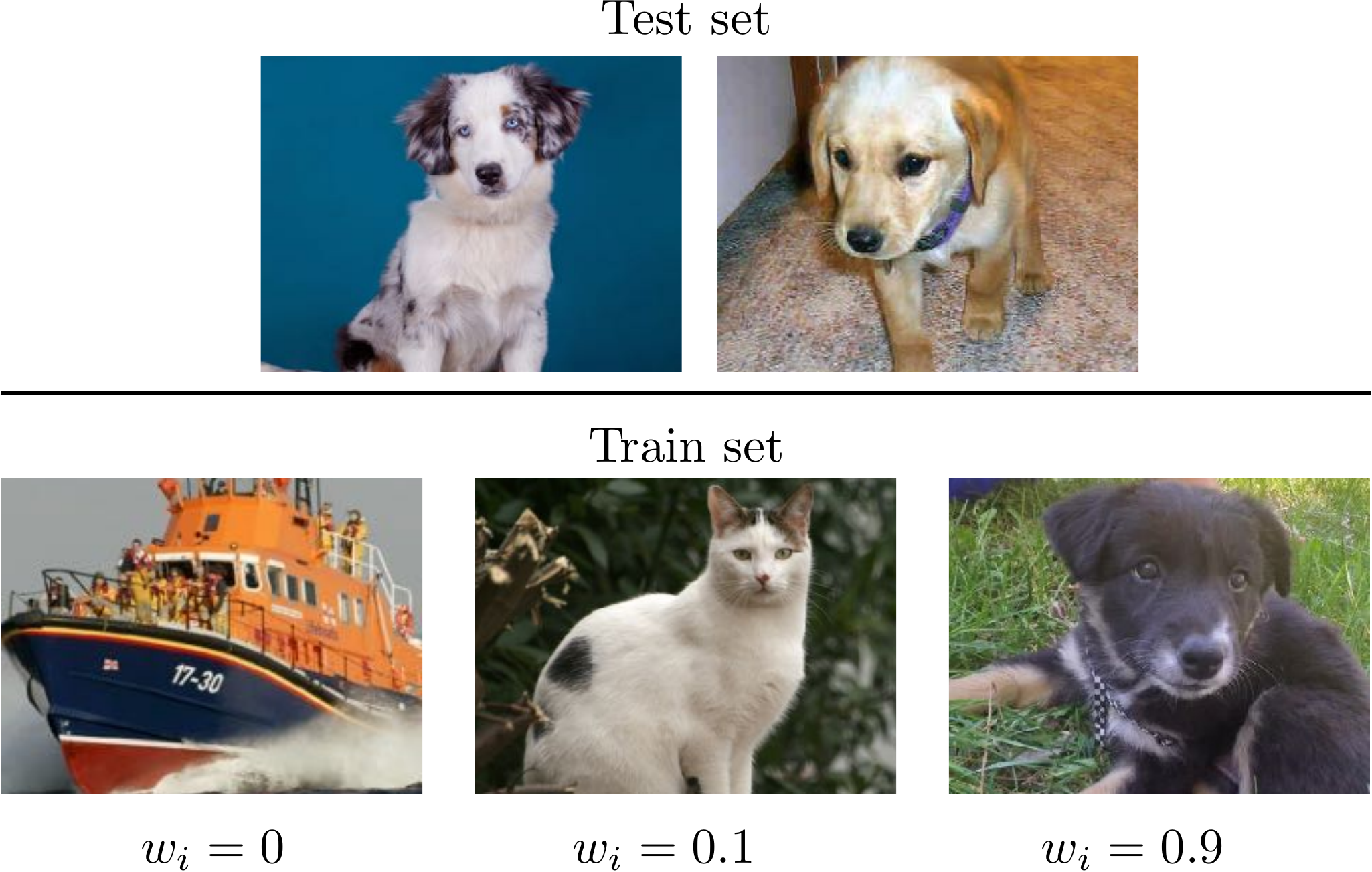}
    \caption{Data reweighting principle: test and train distributions are different. 
    We aim to estimate a weight for each training sample that reflects its contribution to the model's performance on the test set.
    For instance, if the test set only contains dog images, we want a large weight on dog images in the training set, a small weight on other animals which might help the model, and zero weight on irrelevant training images.
    These weights should be learned automatically during training.
    }
    \label{fig:basic_concept}
\end{figure}

This paper aims to provide a detailed analysis of the corresponding joint dynamics. 
Our main result indicates that warm starting with mirror descent leads to \emph{sparse weights}: after training, only a few training points have a non-zero weight, which is detrimental to the generalization power of the corresponding parameters.
Our results show a weakness of warm-started bilevel approaches for this problem and are a first step toward explaining the hardness of data-reweighting.

\textbf{Contributions and paper organization}~In \autoref{sec:problem}, we introduce \textbf{data reweighting as a bilevel problem} and explain how warm-started bilevel aims at solving the problem.
In \autoref{sec:dynamical_system}, we \textbf{study the corresponding dynamics}. 
We focus on two settings.
When the parameters are updated at a much \emph{greater} pace than the weights, we formalize the intuition that this recovers the standard, non-warm-started, bilevel approach, which leads to satisfying solutions but takes a long time to converge. 
In the opposite setting, where the parameters are updated at a much \emph{slower} pace than the weights, where \textbf{we show that this leads to extremely sparse weights}, which in turn hinders the generalization of the model since the model is effectively trained with few samples. 
Finally, \autoref{sec:expe} gives numerical results illustrating the theory presented in the paper.

\textbf{Notation:} The vector containing all $1$'s is $\mathbb{1}_k\in\mathbb{R}^k$.
The simplex is $\Delta_n = \{w\in\mathbb{R}^n_+|\sum_{i=1}^nw_i=1\}$.
The multiplication of two vectors $u, v\in\mathbb{R}^n$ is $u\odot v\in\mathbb{R}^n$ of entries $(u_i v_i)$.
The set $\{1, n\}$ is the set of integers $1$ to $n$.
The support $\mathrm{Supp}(w)$ of a vector $w\in\mathbb{R}^n$ is the set of indices $i$ in $\{1, n\}$ such that $w_i\neq 0$.
Given a set of indices $S=(s_1, \dots, s_l)\subset\{1, n\}$ of size $s$, the restriction to $S$ of a vector $w\in\mathbb{R}^n$ is the vector $w\mid_{S}\in\mathbb{R}^l$ of entries $w_{s_j}$ for $j\in\{1, l\}$.
The gradient (resp. Hessian) of a loss $\ell(\theta; x)$ is $\nabla_\theta \ell(\theta; x)$ (resp. $\nabla^2_{\theta\theta}\ell(\theta; x)$).
The $\mathrm{range}$ of a matrix is the span of its column.
\section{Data reweighting as a bilevel problem}
\label{sec:problem}
We consider a train dataset $X = [x_1, \dots, x_n]$, %
and a testing dataset $X' = [x'_1, \dots, x'_m]$. %
Our goal is to train a machine learning model on the train set that has good performance on the testing set.
Letting $\theta\in\mathbb{R}^p$ the parameters of the model, we let $\ell(\theta;x)$ the loss corresponding to the train set, and $\ell'(\theta; x')$ the loss of the test set. 
The classical \emph{empirical risk minimization} cost function is $G(\theta) = \frac1n\sum_{i=1}^n\ell(\theta;x_i)$ for the train set and $F(\theta) :=\frac1m\sum_{j=1}^m\ell'(\theta;x_j')$ for the test set, where each sample has the same weight.

The goal of data reweighting is to give a different weight from $1/n$ to each training data in order to get a solution $\theta$ that leads to a good performance on the test set, i.e., leads to a small $F(\theta)$.
To this end, we introduce 
\begin{equation}
    \label{eq:inner_loss}
G(\theta, w) := \sum_{i=1}^nw_i\ell(\theta;x_i)
\end{equation}
defined for $w$ belonging to the \emph{simplex} $\Delta_n = \{w\in \mathbb{R}_+^n|\enspace \sum_{i=1}^nw_i=1\}$.
Ideally, we would want $w_i$ to be large when the training sample $x_i$ helps the model's performance on the testing set, and conversely, $w_i$ should be small when $x_i$ does not help the model on the testing set.
We make the following blanket assumption which makes the bilevel problem well-defined:
\begin{assumption}[Strong convexity]
\label{ass:strong_convexity}
    The loss functions $\ell, \ell'$ are differentiable. Additionally, the function $\theta\mapsto \ell(\theta; x)$ is $\mu-$ strongly convex with $\mu >0$ for all $x$, i.e. for all $x, \theta$, we have $\lambda_{\min}(\nabla^2_{\theta\theta}\ell(\theta; x)) \geq \mu$.
\end{assumption}
For instance, such assumption is verified if $\ell(\theta; x)$ is of the form $\mathrm{fit}(\theta; x) + \mu \|\theta\|^2 / 2 $, where the $\mathrm{fit}$ function is a data-fit term that is convex (like a least-squares or logistic loss) and $\mu \|\theta\|^2 / 2$ is a regularizer.

Changing the weights $w$ modifies the cost function $G(\theta, w)$, hence its minimizer is now a function of $w$, denoted $\theta^*(w) = \arg\min_{\theta} G(\theta, w)$. 
Note that the strong-convexity assumption~\ref{ass:strong_convexity} implies that the function $G$ itself is $\mu-$strongly convex, guaranteeing the existence and uniqueness of $\theta^*(w)$.
Data reweighting is formalized as the \emph{bilevel problem} \begin{equation}
\label{eq:bilevel_problem}
    \min_{w\in \Delta_n} h(w):= F(\theta^*(w)),
\end{equation}
where $\theta^*$ depends implicitly on $w$ through $G$.
\subsection{Importance sampling as the hope behind data reweighting?}
One can take a distributional point of view on the problem to get better insights into the sought-after solutions.
We let $\mu = \sum_{i=1}^n \delta_{x_i}$ the empirical training distribution, and $\nu = \sum_{j=1}^m \delta_{x'_j}$ the empirical test distribution.
Instead of having one weight per data sample, we now have a weighting function $\omega:\mathbb{R}^d\to \mathbb{R}_+$ such that $\int\omega(x)d\mu(x) = 1$. The bilevel problem becomes
\begin{align}
    \begin{split}
    \label{eq:measure_bilevel}
        \min_{\omega}\int\ell'(\theta^*(\omega); x)d\nu(x)\text{ subject to} \\
        \theta^*(\omega)\in \arg\min_\theta \int\ell(\theta; x)\omega(x)d\mu(x)    
    \end{split}
\end{align}
In practice, the distributions $\mu$ and $\nu$ are sums of Diracs; we can wonder what happens instead if they are continuous.
\begin{proposition}
    \label{prop:measure_prop}
    If $\nu$ is absolutely continuous w.r.t. $\mu$, and $\ell = \ell'$, a global solution to the bilevel problem~\eqref{eq:measure_bilevel} is $\omega^*(x) = \frac{d\nu}{d\mu}(x)$.
\end{proposition}
This ratio $d\nu/d\mu$ is precisely the one that \emph{importance sampling} techniques~\cite{tokdar2010importance} try to estimate: in this specific case, bilevel optimization recovers importance sampling.
We now turn to the resolution of the bilevel problem.
\subsection{Solving the bilevel problem}
The bilevel optimization problem corresponding to data-reweighting is the \emph{single-level} optimization of the non-convex function $h$ over the simplex $\Delta_n$~\eqref{eq:bilevel_problem}, for which \emph{mirror descent}~\cite{nemirovskij1983problem,beck2003mirror} is an algorithm of choice.
Starting from an initial guess $w^0\in \Delta_n$, mirror descent iterates $\tilde{w}^{k+1} = w^k \odot \exp(-\eta \nabla h(w^k))$ and $w^{k+1} = \frac{\tilde{w}^{k+1}}{\|\tilde{w}^{k+1}\|_1}$, where $\odot$ is the element-wise multiplication.

This method involves the gradient of the value function $h$, which is obtained with the chain rule~\cite{samuel2009learning,domke2012generic}
\begin{equation}
    \label{eq:chain_rule}
\nabla h(w) = \left[\frac{\partial\theta^*}{\partial w}\right]^T \nabla F(\theta^*(w)).
\end{equation}
The implicit function theorem then gives, thanks to the invertibility of $\nabla^2_{\theta\theta}$ guaranteed by \autoref{ass:strong_convexity}:
\begin{equation}
    \label{eq:ift}
    \frac{\partial\theta^*}{\partial w} = - \left[\nabla_{\theta\theta}^2G(\theta^*(w), w)\right]^{-1}\nabla_{\theta w}^2G(\theta^*(w), w).
\end{equation}
In the data reweighting problem, $G$ has a special structure, which gives
$\nabla_{\theta\theta}^2G(\theta, w)=\sum_{i=1}^nw_i\nabla^2\ell(\theta;x_i)$ and $\nabla_{\theta w}^2G(\theta, w) = [\nabla \ell(\theta, x_1), \dots, \nabla \ell(\theta, x_n)]\in\mathbb{R}^{p\times n}$.
We finally obtain $\nabla h(w) = \Psi(\theta^*(w), w)\in\mathbb{R}^n$ where the $i-$th coordinate of $\Psi$ is given by 
\begin{equation}
    \label{eq:hyper_gradient}
\Psi(\theta, w)_i = -\left\langle \nabla \ell(\theta;x_i), \left[\nabla_{\theta }^2G(\theta, w)\right]^{-1}\nabla F(\theta)\right\rangle.
\end{equation}
This hyper-gradient has an intuitive structure: letting $\langle u, v\rangle_{\theta}$ the scalar product defined over $\mathbb{R}^p$ by $\langle u, v\rangle_{\theta} = \langle u, \left[\nabla_{\theta \theta}^2G(\theta, w)\right]^{-1}v\rangle$, the hyper-gradient corresponding to sample $i$ is simply the (opposite) of the alignment measured with this new scalar product between the gradient of the sample $i$, and the gradient of the outer function $F$.
Therefore, this gradient increases weights for which $\nabla \ell(\theta; x_i)$ aligns with the outer gradient.
\begin{figure}[t!]
  \begin{algorithm}[H]
  \caption{Exact Bilevel Algorithm}
  \label{algo:mirrordescent}
    \begin{algorithmic}
       \STATE {\bfseries Input :} Initial point $w^0\in \Delta$, step-size $\eta$, number of iterations $N$.
       \FOR{$k=1$ {\bfseries to} $N$}
       \STATE Compute $\theta^*(w^k)$
       \STATE Compute $\nabla h(w^k) = \Psi(\theta^*(w^k), w^k)$
       \STATE Update $w^{k+1} = \dfrac{w^k \odot \exp(-\eta\nabla h(w^k))}{\|w^k \odot \exp(-\eta \nabla h(w^k))\|_1}$
        \ENDFOR
         \STATE \textbf{Return : }$w^N$.
        \end{algorithmic}
  \end{algorithm}
  \vspace{-2em}
\end{figure}
The mirror descent algorithm to solve the bilevel problem~\eqref{eq:bilevel_problem} is described in \autoref{algo:mirrordescent}.

Unfortunately, this algorithm is purely theoretical as it is not implementable; indeed, computing $\theta^*(w^k)$ is equivalent to finding the minimum of $G$, which is generally impossible. 
We can instead try to approximate it by replacing $\theta^*(w^k)$ by the output of many iterations of an optimization algorithm, but then the cost of the algorithm becomes prohibitive: each iteration requires the approximate resolution of an optimization problem.
\subsection{The practical solution: warm-started bilevel optimization}
A sound idea for a scalable algorithm is to use an iterative method instead to minimize $G$ and have $\theta$ and $w$ evolve simultaneously.
In this case, we have two sets of variables $\theta^k, w^k$.
The parameters $\theta^k$ are updated using standard algorithms like (stochastic) gradient descent, and then to update $w^k$, we approximate the hyper-gradient by $\nabla h(w^k)\simeq \Psi(\theta^k, w^k)$ in Eq.~\eqref{eq:hyper_gradient}.
For simplicity, we use gradient descent with step-size $\rho$ to update $\theta^k$i, and each iteration consists of one update of $\theta^k$ followed by one update of $w^k$.
The full procedure is described in \autoref{algo:scalablemirrordescent}.
\begin{figure}[t]
  \begin{algorithm}[H]
  \caption{Warm Started Bilevel Algorithm}
  \label{algo:scalablemirrordescent}
    \begin{algorithmic}
       \STATE {\bfseries Input :} Initial points $\theta^0\in\mathbb{R}^p$ and $w^0\in \Delta_n$, step-sizes $\eta$ and $\rho$, number of iterations $N$.
       \FOR{$k=1$ {\bfseries to} $N$}
       \STATE Compute $\Psi(\theta^k, w^k)$
       \STATE Update $\theta^{k+1} = \theta^k - \rho \nabla G(\theta^k, w^k)$
       \STATE Update $w^{k+1} = \dfrac{w^k \odot \exp(-\eta \Psi(\theta^k, w^k) )}{\|w^k \odot \exp(-\eta \Psi(\theta^k, w^k))\|_1}$
        \ENDFOR
         \STATE \textbf{Return : }$w^N$.
        \end{algorithmic}
  \end{algorithm}
  \vspace{-2em}
\end{figure}
As a side note, this algorithm can still be expensive to implement because i) computing $\Psi$ requires solving a linear system and ii) in a large-scale setting when $n$ and $m$ are large, computing the inner and outer gradients and the Hessian scales linearly with the number of samples.
Several works try to fix these issues by proposing bilevel algorithms that do not have to form a Hessian or invert a system and that are stochastic, i.e., only use one sample at each iteration to progress~\cite{ghadimi2018approximation,yang2021provably,grazzi2021convergence,li2022fully}.
For our theoretical analysis, we do not consider such modifications and focus on the bare-bones case of \autoref{algo:scalablemirrordescent}, which slightly departs from practice, but is already insightful.
\begin{figure*}[t!]
    \centering
 \includegraphics[width=.99\textwidth]{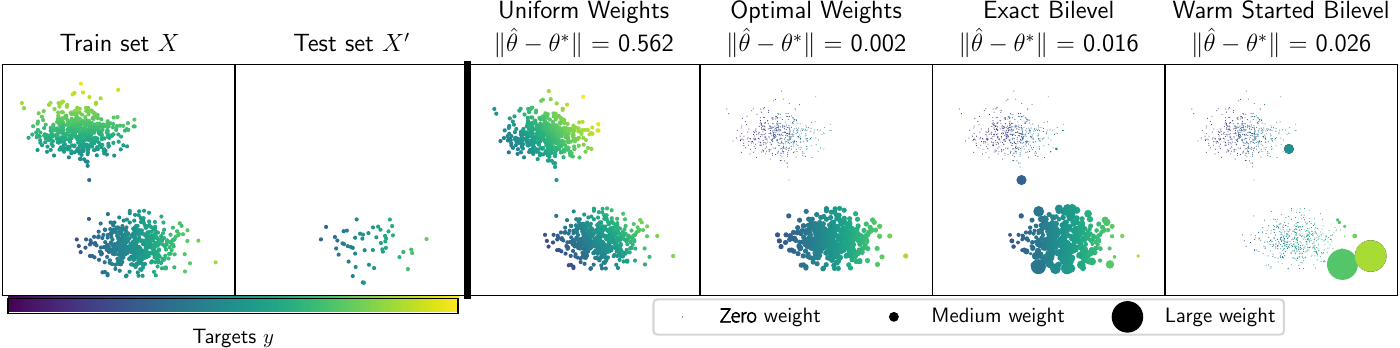}
     \caption{\textbf{Data reweighting for a linear regression problem.} The first two figures show the problem setup. The train set consists of points from two clusters. Points color corresponds to the value of target $y = \langle \hat{\theta}_i, d\rangle +\mathrm{noise}$, where each cluster has its own linear model $\hat{\theta}_1$ and $\hat{\theta}_2$. 
     The test set consists of points only from the first cluster with $\hat{\theta}=\hat{\theta}_1$.
     The last four figures show four reweighting solutions.
     The size of the points in the picture corresponds to the weight value $w_i$ for the corresponding sample. 
     We give a small non-zero radius to zero weights for visualization purposes.
     \textbf{Uniform weights:} same weights for all data points regardless of the cluster. Since the full dataset is not separable with a linear model, the linear model finds a compromise that leads to a poor solution. \textbf{Optimal weights:} equal weights for points from the first cluster, zero weights for the other points.
     The wrong cluster is fully discarded, leading to optimal estimation of the parameters.
     \textbf{Exact bilevel:}  output of \autoref{algo:mirrordescent}, gets to a solution close to the optimal weights, and the cluster is correctly identified with almost uniform weights.
     \textbf{Warm started bilevel:} output of \autoref{algo:scalablemirrordescent}, gets to a very sparse solution which does not generalize well.}
     \label{fig:illustration} 
 \end{figure*}
\subsection{A toy experiment}
\label{subsec:toy_expe}
We describe a toy experiment illustrating the practical issues of \autoref{algo:scalablemirrordescent}.
We consider a 2-d regression setting, where the train data $[x_1, \dots, x_n]$ comes from a mixture of two Gaussian, and each cluster has a different regression parameter.
Formally, given two centroids $\mu_1, \mu_2\in\mathbb{R}^2$ and parameters $\hat{\theta}_1, \hat{\theta}_2\in\mathbb{R}^2$, each training sample $x_i = [d_i, y_i]$ is generated by
\begin{align*}
z_i\sim \mathrm{Bern}(1, 2), \enspace d_i \sim \mathcal{N}(\mu_{z_i}, I), y_i\sim \mathcal{N}(\langle d_i, \theta^*_{z_i}\rangle , \sigma^2),
\end{align*}
where $\mathrm{Bern}(1, 2)$ is the Bernoulli law with probability $1/2$ over $\{1, 2\}$.
Here, $z_i$ represents the random cluster associated to $x_i$.
Meanwhile, the test data $x_j'$ is only drawn from the first cluster with parameter $\hat{\theta} = \hat{\theta}_1$.
The corresponding points $d$ are plotted in the two leftmost figures in \autoref{fig:illustration} in a color reflecting the target's value $y$.
The model is linear with a least-squares loss, meaning that the train and test loss are
\begin{equation}
    \ell(x, \theta) = \ell'(x, \theta) =  \frac12\|\langle d, \theta\rangle - y\|^2\text{ with }x=[d, y].
\end{equation}
The four rightmost figures in \autoref{fig:illustration} display four possible reweighting solutions.
Taking uniform weights $w_i = 1/n$, the train loss is minimized by finding a compromise between the two cluster's parameters, and the model does not recover $\hat{\theta}_1$ nor $\hat{\theta}_2$. 
This leads to a high training and testing loss, and $\|\hat{\theta} - \theta^*\|$ is large.
Intuitively, in the data reweighting framework, we want the weights to focus only on the first cluster and to discard points from the second cluster, which is irrelevant to the test problem.
Additionally, in a noiseless setting ($\sigma=0$), setting $w_i$ proportional to $\delta_{z_i=1}$ leads to a global minimum in the bilevel problem since $\theta^*(w) = \hat{\theta}$ and hence $h(w) = F(\hat{\theta}) = 0$.
In this ideal scenario, the model is fit only on the correct cluster, leading to a small error $\|\hat{\theta} - \theta^*\|$, due only to the label noise.

We then apply \autoref{algo:mirrordescent} to this problem.
Here, since the inner problem is quadratic, we have a closed-form $\theta^*(w) = \left(\sum_{i=1}^nw_id_id_i^T\right)^{-1}\sum_{i=1}^nw_iy_id_i$, making \autoref{algo:mirrordescent} implementable.
It gives a reasonable solution where most of the weights in the second cluster are close to $0$, and weights in the first cluster are nearly uniform.
Thus, the error $\|\hat{\theta} - \theta^*\|$ is small.

We finish with the warm-started \autoref{algo:scalablemirrordescent}.
We observe that the algorithm outputs sparse weights: only a few weights are non-zero.
This, in turn, leads to a higher estimation error $\|\hat{\theta} - \theta^*\|$.
As an important note, this effect is worsened when taking a higher learning rate $\eta$ for the update of $w$, and is mitigated by taking small learning rates $\rho$ that lead to slow convergence to better solutions.
This illustrates a problem with the warm-started method: it has a trajectory distinct from the exact bilevel method and recovers sub-optimal solutions.
In the following section, we develop a mathematical framework that explains this phenomenon.
\section{A Dynamical System View on the Warm-Started Bilevel Algorithm}
\label{sec:dynamical_system}
The study of iterative methods like \autoref{algo:scalablemirrordescent} is notoriously hard; we focus on the  dynamical system obtained by letting the step sizes $\eta$, $\rho$ go to $0$ at the same speed.
\autoref{algo:scalablemirrordescent} can be seen as the discretization of the Ordinary Differential Equation (ODE)
\begin{equation}
    \label{eq:bilevel_flow}
\begin{cases}
\dot{\theta} & = - \alpha \nabla G(\theta, w)\\
\dot w  & = - \beta P(w)\Psi(\theta, w)
\end{cases}\enspace,
\end{equation}
where $\alpha>0$ controls the speed of convergence of $\theta$, $\beta>0$ controls the speed of $w$, and $P(w) = \mathrm{diag}(w) - ww^T$ is a preconditioning matrix, which is the inverse Hessian metric $\mathrm{diag}(w)$ associated to the entropy applied to the projector on the tangent space $I_n - \mathbb{1}_nw^T$, recovering a Riemannian gradient flow~\cite{gunasekar2018characterizing}.
This ODE can be recovered from \autoref{algo:scalablemirrordescent} in the following sense:
\begin{proposition}
    \label{prop:discretization_ode}
    Let $(\theta(t), w(t))$ the solution of the ODE~\eqref{eq:bilevel_flow}, and $(\theta^k, w^k)$ the iterates of \autoref{algo:scalablemirrordescent}.
    Assume that the steps are such that $(\rho, \eta) = \tau(\alpha, \beta)$ with $\tau >0$.
    Then, for any $T >0$, we have $\lim_{\tau \to 0}\sup_{t\in[0, T]}\|(\theta(t), w(t)) - (\theta^{\lfloor t /\tau \rfloor}, w^{\lfloor t /\tau \rfloor})\| = 0$.
\end{proposition}
When $\alpha \gg \beta$, the dynamics in $\theta$ are much faster, and we recover the classical bilevel approach:
\begin{theorem}
    \label{prop:recover_bilevel}
    Let $w^*(t)$ the solution of the ODE $\dot w = -P(w)\nabla h(w)$, and $\theta^{\alpha, \beta}, w^{\alpha, \beta}$ the solution of \eqref{eq:bilevel_flow}. Then for all $\alpha$, for all time horizon $T$, we have  
    $$\lim_{\beta\to 0}\sup_{t\in[0, T]}\|w^{\alpha, \beta}(t / \beta) - w^*(t)\|=0$$
    $$\lim_{\beta\to 0}\|\theta^{\alpha, \beta}(t) - \theta^*(w(t ))\|\leq \|\theta_0 - \theta^*(w_0)\|e^{-\mu \alpha t} $$
\end{theorem}
The proof of this result uses the classical tools in bilevel literature~\cite {ghadimi2018approximation}.
This result highlights that, as expected, if $\alpha \gg \beta$, the variable $\theta$ is tracking $\theta^*(w)$, which in turn means that $w$ follows the direction of the true gradient of $h$: we recover the Exact Bilevel dynamics, that corresponds to the gradient flow of $h$.
The rest of this section is devoted to understanding what happens in the other regime where $\beta \gg\alpha$, when the dynamics in $w$ are much faster than that in $\theta$. 
\subsection{Mirror descent flows on the simplex}
Before understanding the joint dynamics, we consider the dynamics of the warm-started bilevel ODE~\eqref{eq:bilevel_flow} when only $w$ evolves, i.e., when $\alpha=0$ and $\theta(t) = \theta^0$.
For a vector field $\phi:\mathbb{R}^n\to\mathbb{R}^n$, we consider mirror descent updates given by $\tilde{w}^{k+1} = w^k \odot \exp(-\eta \phi(w^k))$ and $w^{k+1} = \frac{\tilde{w}^{k+1}}{\|\tilde{w}^{k+1}\|_1}$.
Letting the step size $\eta$ to $0$, we recover the \emph{mirror descent flow}~\cite{gunasekar2021mirrorless}, that is the ODE 
\begin{equation}
\label{eq:mirror_flow}    
\dot{w} = - \Phi(w)\text{ with }\Phi(w) = P(w) \phi(w).
\end{equation}
When the vector field $\phi$ is the gradient of a function $f:\mathbb{R}^n\to \mathbb{R}$, we recover mirror descent to minimize $f$ on the simplex, with standard guarantees. 
In our warm-started formulation, the field $\phi(w) = \Psi(\theta_0, w)$ does not correspond to a gradient since its Jacobian may not be symmetric. 
We analyze the stationary points of this ODE and their stability.
We recall that the support of a vector $w$, $\mathrm{Supp}(w)$, is the set of indices in $\{1, n\}$ such that $w_i \neq 0$, and let $l$ its cardinal.
\begin{proposition}
    \label{prop:stat}
    The stationary points of the mirror descent flow \eqref{eq:mirror_flow} are the $w$ such that $\phi(w)\mid_{\mathrm{Supp}(w)}$ is proportional to $\mathbb{1}_l$.
\end{proposition}
All vectors of the form $w = (0, \dots, 1, \dots, 0)$ fulfill this condition, but there might be other solutions with a larger support.
To study the stability of these points, we turn to the Jacobian of $\Phi$. 
Letting for short $\phi = \phi(w)\in\mathbb{R}^n$ and $J = D\phi(w)\in\mathbb{R}^{n\times n}$, we find 
$$
D\Phi(w) = \mathrm{diag}(\phi) + \mathrm{diag}(w)J - \langle w, \phi\rangle I_n - w(\phi^T + w^T J)
$$
Because of the simplex constraint, we only care about this Jacobian in directions that are in the tangent space $T_\Delta = \{\delta \in \mathbb{R}^n|\enspace \sum_{i=1}^n \delta_i = 0\}$.
Since $\Phi$ makes the flow stay in the simplex, we have that $D\Phi(w)[\delta] \in T_{\Delta}$ for any $\delta\in T_\Delta$.
Without loss of generality, we assume that $w_1, \dots, w_l\neq 0$ and $w_{l+1} = \dots =w_n=0$.
The Jacobian greatly simplifies for coordinates $i$ such that $w_i = 0$; indeed in this case for any vector $\delta$ in $\mathbb{R}^n$ we have $D\Phi(w)[\delta]_i = (\phi_i - \langle w, \phi\rangle)\delta_i$. On the other hand, for a coordinate $i$ in the support, under the optimality conditions, taking a displacement of the form $\delta = (\tilde{\delta}_l, 0, \dots, 0)$, we find that $D\Phi(w)[\delta]_i = w_i( [\tilde{J}\tilde{\delta}]_i - \sum_{j=1}^l(\phi_j + [\tilde{J}^Tw]_j)\tilde{\delta}_j)$, where $\tilde{J}$ is the upper-left $l\times l$ block of $J$. In other words, $D\Phi(w)$ has the following structure when $w$ is a stationary point of the flow:
$$
D\Phi(w) = \begin{bmatrix}
    P(\tilde{w})\tilde{J} & (*) \\
    0 & \mathrm{diag}(\phi_i - \langle w, \phi\rangle)
\end{bmatrix}
$$
We readily obtain the stability condition, which is that this matrix has all positive eigenvalues:
\begin{proposition}
    \label{prop:stable_flow}
    A stationary point $w$ of \eqref{eq:mirror_flow} is stable if and only if for all $i\notin\mathrm{Supp}(w)$ we have $\phi(w)_i > \langle w, \phi(w)\rangle$ and the matrix $P(\tilde{w})\tilde{J}$, as a linear operator $T_\Delta \to T_\Delta$, has eigenvalues with positive real parts. 
\end{proposition}

We can finally quantify the local speed of convergence towards these stationary points:
\begin{proposition}
    \label{prop:flow_speed}
    Let $w^*$ be a stable stationary point of \eqref{eq:mirror_flow}, and let $\delta$ be in the tangent cone at $w^*$.
    Then, letting $w(t)$ the trajectory of the ODE \eqref{eq:mirror_flow} starting from $w^* + \delta$, we have $w(t) = w^* + \exp(-D\Phi(w^*)t)\delta + o(\delta)$.
\end{proposition}
This result is classical in ODE theory~\cite{chicone2006ordinary}.
We now turn to the case of interest for bilevel optimization, where $\phi$ corresponds to the hyper-gradient field with frozen parameters $\theta$, i.e., $\phi(w) = \Psi(\theta_0, w)$.
\subsection{The Bilevel Flow with Frozen Parameters}
\label{subsec:dyn}
\begin{figure}[t]
\centering
\includegraphics[width=.9\columnwidth]{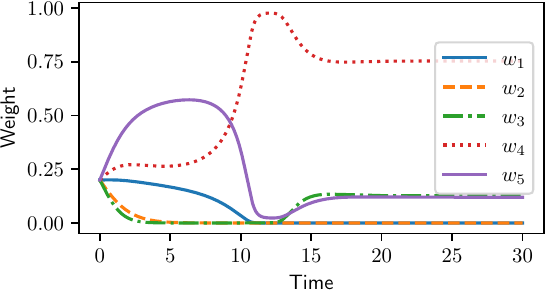}
    \caption{Weights dynamics from the mirror-flow ODE~\eqref{eq:mirror_flow} with the field $\phi$ in Eq.~\eqref{eq:frozen_field}, with $n=5$ training samples and $p=3$ parameters. 
    Matrix $\Gamma$ is randomly sampled with i.i.d. Gaussian entries, same for $\nabla F(\theta_0)$, and the individual Hessians are taken as $\nabla^2_{\theta\theta}\ell(\theta; x_i) = u_iu_i^T + 0.1I_p$ where the $u_i$ are drawn i.i.d. from random Gaussians.
    The weights follow a non-trivial dynamic that eventually converges to a solution with at most $p$ non-zero weights, as predicted by \autoref{prop:sparse_weights_frozen}.
    Most random initializations lead to only one non-zero weight; we display a rarer dynamic here.}
\label{fig:weights_dynamic}
\end{figure}
With fixed parameters $\theta_0$, the field of interest $\phi(w) = \Psi(\theta_0, w)$ is given by the simple equation
\begin{equation}
    \label{eq:frozen_field}
    \phi(w) = \Gamma g(w),
\end{equation}
where $\Gamma =[\nabla \ell(\theta_0; x_1),\dots, \nabla \ell(\theta_0; x_n)]\in\mathbb{R}^{n\times p}$ is a matrix containing all the inner gradients, and $g(w) = \left(\sum_{i=1}^n w_i \nabla^2_{\theta}\ell(\theta_0; x_i)\right)^{-1}\nabla F(\theta_0)\in\mathbb{R}^p$ is a the outer gradient $\nabla F(\theta_0)$ transformed by the metric induced by the Hessian of the inner function $G$.
\autoref{fig:weights_dynamic} illustrates the weights' dynamics with few samples on a low-dimensional problem and its ``sparsifying'' effect.
The field $\phi$ only depends mildly on $w$, through the impact of $w$ on the Hessian of the inner problem.
We start our analysis with the simple case where all Hessians $\nabla_\theta^2\ell(\theta_0, x_i)$ are the same, in which case $g$ is constant:
\begin{proposition}
    \label{prop:constant_frozen_flow}
    If $g(w)$ is constant in Eq.~\ref{eq:frozen_field}, then the field $\phi(w)$ is constant, equal to $\phi\in\mathbb{R}^n$. The mirror flow~\eqref{eq:mirror_flow} converges to $w^*$ such that $w^*_i=0$ if $i\notin\arg\max \phi$, and $w^*_i$ is proportional to $w^*_0$ otherwise.
\end{proposition}
In particular, if $\phi$ only has a unique maximal coefficient, the mirror flow converges to a $w^*$ with only \emph{one} non-zero weight -- which is as sparse as it gets.
We now turn to the general case with non-constant Hessians.

The vector field $\phi$ has a ``low rank'' structure, as it is parameterized by $g(w)$ of dimension $p$, which generally is much lower than $n$.
It is, therefore, natural that it is hard to satisfy the stability conditions of \autoref{prop:stat} with non-sparse weights: we have $n$ conditions to verify with a family of vectors $g(w)$ that is of dimension $p$. 
To formalize this intuition, we define the following set:
\begin{definition}
    \label{def:feasible_set}
    The set $\mathcal{I}_l^p$ is the set of $l\times p$ matrices $Z$ such that $\mathbb{1}_l\in \mathrm{range}(Z)$ or $\mathbb{0}_l\in \mathrm{range}(Z)$.
\end{definition}
In other words, this is the set of matrices $Z$ for which the equation $Zx = \mathbb{1}_l$ or the equation $Zx = \mathbb{0}_l$ has a non-zero solution $x$.
This set either contains most matrices if $l \leq p$ or very few if $p < l$:
\begin{proposition}
    \label{prop:set_proba}
    Assume that $Z\in \mathbb{R}^{l\times p}$ has i.i.d. entries drawn from a continuous distribution. If $l \leq p$, then $\mathbb{P}(Z\in \mathcal{I}_l^p) = 1$, while if $l > p$ then $\mathbb{P}(Z\in \mathcal{I}_l^p) = 0$.
\end{proposition}
This set is linked to the stationary points of the flow~\eqref{eq:mirror_flow}:
\begin{proposition}
\label{prop:sparse_weights_frozen}
    If $w$ is a stationary point of~\eqref{eq:mirror_flow} with the hypergradient~\eqref{eq:frozen_field}, letting $l$ the size of the support of $w$, we have $\Gamma|_{\mathrm{Supp}(w)}\in \mathcal{I}_l^p$.
\end{proposition}
This proposition immediately shows that, in general, we cannot find a stationary point of~\eqref{eq:mirror_flow} with support larger than $p$, the number of features: the mirror descent dynamics applied with the hypergradient, if it converges, will converge to a sparse solution with at most $p$ non-zero coefficients.
Note that we could not show that the corresponding flow always converges; our results only indicate that \emph{if} the flow converges, then it must be towards a sparse solution.
Through numerical simulations, we have identified trajectories of the ODE that oscillate and do not seem to converge.
\subsection{Two Variables Dynamics}
We conclude our analysis by going back to the two variables ODE~\eqref{eq:bilevel_flow} and use the previous analysis to show that the sparsity issue also impacts the warm-started bilevel problem in the case where $\beta \gg \alpha$.
To do so, we first assume that the mirror flow ODE with frozen parameters $\theta$ in Eq~\eqref{eq:mirror_flow} converges for all $\theta$.
\begin{assumption}
\label{ass:omega}
    The ODE $\dot{w} = - \Psi(\theta, w)$ starting from $w_0$, with fixed $\theta$, is such that $w(t)$ goes to a limit as $t$ goes to infinity. 
    We call $\Omega(\theta, w_0)$ this limit.
\end{assumption}
Note that, following \autoref{prop:sparse_weights_frozen}, the limit $\Omega$ is in general sparse with a support smaller than $p$.
We now give a result similar to \autoref{prop:recover_bilevel} but when the dynamics in $w$ gets much faster than that in $\theta$:
\begin{theorem}
    \label{prop:sparse_dynamics}
    Let $\theta^*(t)$ the solution of the ODE $\dot \theta = -\nabla G(\theta, \Omega(\theta, w_0))$, and $\theta^{\alpha, \beta}, w^{\alpha, \beta}$ the solution of \eqref{eq:bilevel_flow},where $w$ starts from $w_0$. Under technical assumptions described in Appendix, for all $\beta$, for all time horizon $T$, we have  
    $$\lim_{\alpha\to 0}\sup_{t\in[0, T]}\|\theta^{\alpha, \beta}(t / \alpha) - \theta^*(t)\| = 0.$$
\end{theorem}
As a consequence, the parameters $\theta$ track the gradient flow of $G$ obtained with the sparse weights $\Omega(\theta, w_0)$.
This leads to sub-optimal parameters, which are estimated using only a few training samples, and explains the behavior observed in \autoref{fig:illustration}.
Note that our theory only works in the regime where $\alpha \gg \beta$ (\autoref{prop:recover_bilevel}) or $\alpha \ll \beta$ (\autoref{prop:sparse_dynamics}), the behavior of the warm-started bilevel in practice is therefore interpolating between these two regimes.
However, in practice, we observe that the warm-started bilevel methods are attracted to sparse solutions, hinting at the fact that \autoref{prop:sparse_dynamics} might better describe reality than \autoref{prop:recover_bilevel}.
\section{Experiments}
\label{sec:expe}
All experiments are run using the Jax framework~\cite{jax2018github} on CPUs.
\subsection{The role of mirror descent}
We place ourselves in the same toy setup with a mixture of two Gaussians that is described in \autoref{subsec:toy_expe}. We take $n = 500$ and $m = 100$ in dimension $2$. 
This experiment aims to understand if mirror descent is critical to observing the behavior described in the paper.
We use another method to enforce the simplex constraint, namely, we introduce the following inner function: $G(\theta, \lambda) = \sum_{i=1}^n \frac{\sigma(\lambda_i)}{\sum_j \sigma(\lambda_j)} \ell(\theta;x_i)$, where $\sigma$ is the sigmoid function.
Thus, we take the vector $\lambda$ as outer parameters, which defines the weights as $w_i = \frac{\sigma(\lambda_i)}{\sum_j \sigma(\lambda_j)}$ and we use gradient descent on the $\lambda$ instead of mirror descent on $w$ to solve the bilevel problem. We use $F(\theta) =\frac1m\sum_{j=1}^m\ell'(\theta;x_j')$ as outer function. 
We use the same algorithm as \autoref{algo:mirrordescent} but without the mirror step since the $\lambda$'s are unconstrained.
\begin{figure}[t]
\centering
\includegraphics[width=.7\columnwidth]{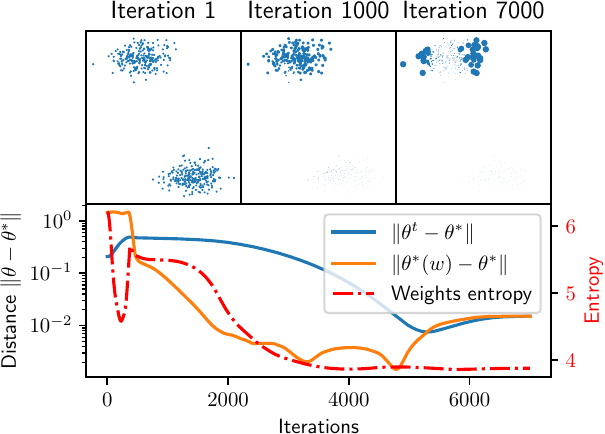}
\caption{Convergence of \autoref{algo:scalablemirrordescent} on a toy mixture problem as described in \autoref{subsec:toy_expe}.
We replace mirror descent by reparameterizing the weights as the softmax of new parameters.
\textbf{Top:} Weights at different iterations.
\textbf{Bottom:} training dynamics.}
\label{fig:gmm_softmax}
\end{figure}
\autoref{fig:gmm_softmax} displays the results.
In blue, we display the error between the parameters and the target $\theta^*$.
In orange, we display the error between the parameters found by minimizing the inner function with the current weights.
Finally, the red curve tracks the entropy of the weights defined as $\mathrm{entropy}=-\sum_{i=1}^n w_i\log(w_i)$.
We use entropy as a proxy for sparsity: entropy is maximized when the weights are uniform and minimized when all weights but one are $0$.
Entropy decreases during training, as suggested by our theory.
Between iterations 1000 and 5000, the weight distribution is not yet sparse and correctly identifies the good cluster; hence minimizing the train loss with those weights leads to good results: the orange curve is low. 
However, because of the warm started dynamics, the parameters take some time to catch up, eventually converging only when the weights are already sparse, leading to a large error.
Here, replacing mirror descent with reparameterization leads to the same behavior described in the paper.
\subsection{Hyper data-cleaning}
\begin{figure}[t]
\centering
\includegraphics[width=.49\columnwidth]{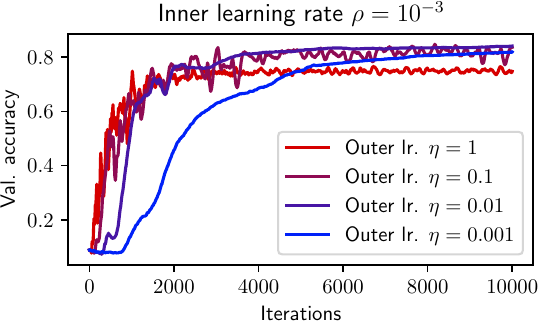}
\includegraphics[width=.49\columnwidth]{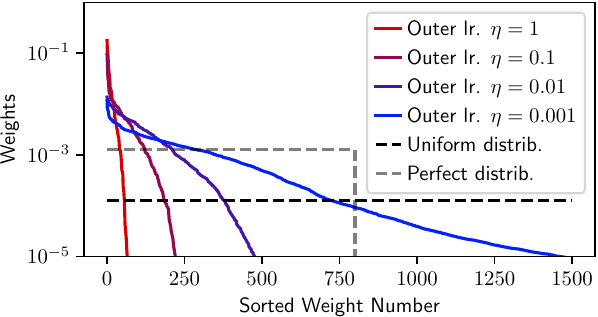}
\caption{Applying SOBA, a scalable warm-start bilevel method, for datacleaning on the MNIST dataset, with a fixed inner learning rate and different outer learning rates.
\textbf{Top:} training curves. 
\textbf{Bottom:} Weights after training, sorted in descending order, zooming in on the first $1500$ weights out of $8000$.
}
\label{fig:mnist_train_curves}
\end{figure}
We conduct a hyper data-cleaning experiment in the spirit of~\cite{franceschi2017forward}.
It is a classification task on the MNIST dataset~\cite{lecun2010mnist}, where the training and testing set consists of pairs of images of size $28\times 28$ and labels in $\{0, 9\}$.
The gist of this experiment is that the training set is \emph{corrupted}: with a probability of corruption $p_c=0.9$, each training sample's label is replaced by a random, \emph{different} label; corrupted samples always have an incorrect label.
The testing set is uncorrupted.
We take $n=8K$ training samples; hence we only have $800$ clean training samples hidden in the training set.

We use a linear model and a cross-entropy loss, with a small $\ell_2$ regularization of $10^{-2}$ on the training loss.
We are, therefore, in the strongly convex setting described in this paper: for any weight $w$, the inner problem has one and exactly one set of optimal parameters.
However, in this case, \autoref{algo:mirrordescent} is not implementable since we do not have a closed-form solution to the regularized multinomial logistic regression.

The linear system resolution in \autoref{algo:scalablemirrordescent} is also impractical; we, therefore, use the scalable algorithm SOBA~\cite{dagreou2022framework}, which has an additional variable that tracks the solution to the linear system updated using Hessian-vector products.
We first run the algorithm with a fixed small inner learning rate $\rho = 10^{-3}$ and several outer learning rates $\eta$. \autoref{fig:mnist_train_curves} displays the results.
The validation loss is computed on samples that are not part of the test set.
When the outer learning rate $\eta$ is too high (red curves), as predicted by our theory, the weights go to very sparse solutions, leading to sub-optimal performance.
When the outer learning rate $\eta$ is too small, the system converges to a good solution, but slowly (the bluest curve has not yet converged).
Overall, the range of learning rates where the algorithm converges quickly to a good solution is very narrow.

Then, we run take different choices of inner learning rate $\rho$ and outer learning rate $\eta$.
We cover a large range of learning rate ratios $r=\eta/\rho$ while ensuring the algorithm's convergence.
Hence, for a fixed target ratio $r$, we pick $\eta$ and $\rho$ so that neither go over a prescribed value $\eta_{\max}=1$ and $\rho_{\max}=10^{-2}$, by taking $\eta = \min(r  \rho_{\max}, \eta_{\max})$ and $\rho = \eta / r$.
The algorithm always converges in this setup.
We perform $12K$ iterations of the algorithm and then compute two metrics: validation accuracy and entropy of the final weights.
As baselines for entropy, we compute the entropy of the \emph{uniform distribution}, equal to $\log(n)$, and of the \emph{perfect distribution}, which would put a weight of $0$ on all corrupted data points, and uniform weight elsewhere, equal to $(1-p_c)\log(n)$.
\begin{figure}[t]
    \centering
\includegraphics[width=.7\columnwidth]{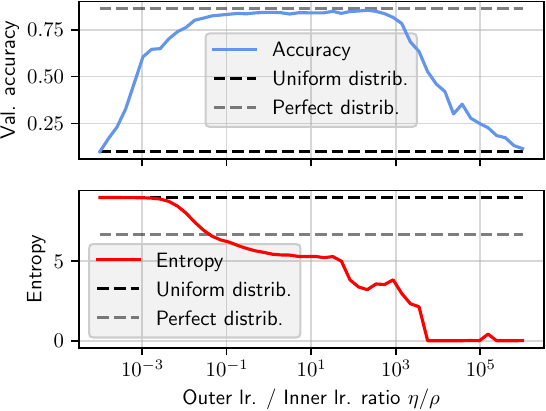}
\caption{Output of SOBA on the datacleaning task for a wide range of learning rate ratios. Entropy is a proxy for weight's sparsity.}
\label{fig:mnist_linear}
\end{figure}
\autoref{fig:mnist_linear} displays the results.
We see that when the outer learning rate $\eta$ is too small, nothing happens because the weights have not yet converged. 
The corresponding accuracy is around $10\%$: training the model on corrupted data completely fails.
Then, when the ratio gets larger, reweighting starts working: for ratios between $10^{-1}$ to $10^2$, the weights learned are not sparse and correctly identify the correct data.
This leads to an accuracy close to the accuracy of the model trained on the perfect distribution, i.e., on the $800$ clean train samples.
However, the time it takes to convergence is roughly linear with the ratio: taking a ratio of $10^{-1}$ leads to training about $1000$ times slower than with $10^{2}$.
Finally, when the ratio is much higher, we arrive at the region predicted by our theory (\autoref{prop:sparse_dynamics}): weights become extremely sparse, as highlighted by the entropy curve, and accuracy decreases. 
We even reach a point where the entropy goes to $0$, i.e., only one weight is non-zero.
The corresponding accuracy is not catastrophic thanks to warm-starting; it is much better than that of a model trained on that single sample: the model has had time to learn something in the period where the weights were non-zero.
This observation is reminiscent of~\cite{vicol2022implicit}, which also mention that the weights trajectory impacts the final parameters.
\section*{Discussion \& Conclusion}
In this work, we have illustrated a challenge for data reweighting with bilevel optimization: bilevel methods must use warm-starting to be practical, but warm-starting induces sparse data weights, leading to sub-optimal solutions.
To remedy the situation, a small outer learning rate should therefore be used, which might, in turn, lead to slow convergence.
Classical bilevel optimization theory~\cite{arbel2021amortized} demonstrates the convergence of warm-started bilevel optimization to the solutions of the true bilevel problem.
This may seem paradoxical at first and in contradiction with our results.
Two explanations lift the paradox: i)\cite{arbel2021amortized} require the ratio $\frac{\alpha}\beta$ to be smaller than some intractable constant of the problem, hence not explaining the dynamics of the system in the setting where $\alpha$ is not much smaller than $\beta$, which is the gist of this paper, ii) convergence results in bilevel optimization are always obtained as non-convex results, only proving that the gradient of $h$ goes to $0$. 
In fact, for the data reweighting problem, several stationary points of $h$ are sparse (see \autoref{prop:sparse_weights_frozen}).
Hence, our results on the sparsity of the resulting solution can be seen as implicit bias results, where the ODE converges to different solutions on the manifold of stationary points.
Finally, our results are orthogonal to works considering the implicit bias of bilevel optimization~\cite{arbel2021amortized,vicol2022implicit}
These results are based on over-parameterization, implying that the inner problem is not strongly convex and has multiple minimizers, while our results do not require such a structure.
\bibliographystyle{abbrv}
\bibliography{main}

\begin{thebibliography}{10}

\bibitem{arbel2021amortized}
M.~Arbel and J.~Mairal.
\newblock Amortized implicit differentiation for stochastic bilevel
  optimization.
\newblock {\em arXiv preprint arXiv:2111.14580}, 2021.

\bibitem{beck2003mirror}
A.~Beck and M.~Teboulle.
\newblock Mirror descent and nonlinear projected subgradient methods for convex
  optimization.
\newblock {\em Operations Research Letters}, 31(3):167--175, 2003.

\bibitem{bommasani2021opportunities}
R.~Bommasani, D.~A. Hudson, E.~Adeli, R.~Altman, S.~Arora, S.~von Arx, M.~S.
  Bernstein, J.~Bohg, A.~Bosselut, E.~Brunskill, et~al.
\newblock On the opportunities and risks of foundation models.
\newblock {\em arXiv preprint arXiv:2108.07258}, 2021.

\bibitem{jax2018github}
J.~Bradbury, R.~Frostig, P.~Hawkins, M.~J. Johnson, C.~Leary, D.~Maclaurin,
  G.~Necula, A.~Paszke, J.~Vander{P}las, S.~Wanderman-{M}ilne, and Q.~Zhang.
\newblock {JAX}: composable transformations of {P}ython+{N}um{P}y programs,
  2018.

\bibitem{chicone2006ordinary}
C.~Chicone.
\newblock {\em Ordinary differential equations with applications}, volume~34.
\newblock Springer Science \& Business Media, 2006.

\bibitem{dagreou2022framework}
M.~Dagr{\'e}ou, P.~Ablin, S.~Vaiter, and T.~Moreau.
\newblock A framework for bilevel optimization that enables stochastic and
  global variance reduction algorithms.
\newblock {\em Advances in Neural Information Processing Systems},
  35:26698--26710, 2022.

\bibitem{devlin2019bert}
J.~Devlin, M.-W. Chang, K.~Lee, and K.~Toutanova.
\newblock {BERT}: Pre-training of deep bidirectional transformers for language
  understanding.
\newblock In {\em Proceedings of the 2019 Conference of the North {A}merican
  Chapter of the Association for Computational Linguistics: Human Language
  Technologies, Volume 1 (Long and Short Papers)}, pages 4171--4186,
  Minneapolis, Minnesota, June 2019. Association for Computational Linguistics.

\bibitem{domke2012generic}
J.~Domke.
\newblock Generic methods for optimization-based modeling.
\newblock In {\em Artificial Intelligence and Statistics}, pages 318--326.
  PMLR, 2012.

\bibitem{franceschi2017forward}
L.~Franceschi, M.~Donini, P.~Frasconi, and M.~Pontil.
\newblock Forward and reverse gradient-based hyperparameter optimization.
\newblock In {\em International Conference on Machine Learning}, pages
  1165--1173. PMLR, 2017.

\bibitem{ghadimi2018approximation}
S.~Ghadimi and M.~Wang.
\newblock Approximation methods for bilevel programming.
\newblock {\em arXiv preprint arXiv:1802.02246}, 2018.

\bibitem{grazzi2021convergence}
R.~Grazzi, M.~Pontil, and S.~Salzo.
\newblock Convergence properties of stochastic hypergradients.
\newblock In {\em International Conference on Artificial Intelligence and
  Statistics}, pages 3826--3834. PMLR, 2021.

\bibitem{gunasekar2018characterizing}
S.~Gunasekar, J.~Lee, D.~Soudry, and N.~Srebro.
\newblock Characterizing implicit bias in terms of optimization geometry.
\newblock In {\em International Conference on Machine Learning}, pages
  1832--1841. PMLR, 2018.

\bibitem{gunasekar2021mirrorless}
S.~Gunasekar, B.~Woodworth, and N.~Srebro.
\newblock Mirrorless mirror descent: A natural derivation of mirror descent.
\newblock In {\em International Conference on Artificial Intelligence and
  Statistics}, pages 2305--2313. PMLR, 2021.

\bibitem{ji2021bilevel}
K.~Ji, J.~Yang, and Y.~Liang.
\newblock Bilevel optimization: Convergence analysis and enhanced design.
\newblock In {\em International conference on machine learning}, pages
  4882--4892. PMLR, 2021.

\bibitem{lecun2010mnist}
Y.~LeCun, C.~Cortes, C.~Burges, et~al.
\newblock Mnist handwritten digit database, 2010.

\bibitem{li2022fully}
J.~Li, B.~Gu, and H.~Huang.
\newblock A fully single loop algorithm for bilevel optimization without
  hessian inverse.
\newblock In {\em Proceedings of the AAAI Conference on Artificial
  Intelligence}, volume~36, pages 7426--7434, 2022.

\bibitem{luketina2016scalable}
J.~Luketina, M.~Berglund, K.~Greff, and T.~Raiko.
\newblock Scalable gradient-based tuning of continuous regularization
  hyperparameters.
\newblock In {\em International conference on machine learning}, pages
  2952--2960. PMLR, 2016.

\bibitem{mahajan2018exploring}
D.~Mahajan, R.~Girshick, V.~Ramanathan, K.~He, M.~Paluri, Y.~Li, A.~Bharambe,
  and L.~Van Der~Maaten.
\newblock Exploring the limits of weakly supervised pretraining.
\newblock In {\em Proceedings of the European conference on computer vision
  (ECCV)}, pages 181--196, 2018.

\bibitem{nemirovskij1983problem}
A.~S. Nemirovskij and D.~B. Yudin.
\newblock Problem complexity and method efficiency in optimization.
\newblock 1983.

\bibitem{pedregosa2016hyperparameter}
F.~Pedregosa.
\newblock Hyperparameter optimization with approximate gradient.
\newblock In {\em International conference on machine learning}, pages
  737--746. PMLR, 2016.

\bibitem{ren2018reweight}
M.~Ren, W.~Zeng, B.~Yang, and R.~Urtasun.
\newblock Learning to reweight examples for robust deep learning.
\newblock In {\em International conference on machine learning}, pages
  4334--4343. PMLR, 2018.

\bibitem{samuel2009learning}
K.~G. Samuel and M.~F. Tappen.
\newblock Learning optimized map estimates in continuously-valued mrf models.
\newblock In {\em 2009 IEEE Conference on Computer Vision and Pattern
  Recognition}, pages 477--484. IEEE, 2009.

\bibitem{shu19metaweight}
J.~Shu, Q.~Xie, L.~Yi, Q.~Zhao, S.~Zhou, Z.~Xu, and D.~Meng.
\newblock Meta-weight-net: Learning an explicit mapping for sample weighting.
\newblock In {\em NeurIPS}, 2019.

\bibitem{tokdar2010importance}
S.~T. Tokdar and R.~E. Kass.
\newblock Importance sampling: a review.
\newblock {\em Wiley Interdisciplinary Reviews: Computational Statistics},
  2(1):54--60, 2010.

\bibitem{vicol2022implicit}
P.~Vicol, J.~P. Lorraine, F.~Pedregosa, D.~Duvenaud, and R.~B. Grosse.
\newblock On implicit bias in overparameterized bilevel optimization.
\newblock In {\em International Conference on Machine Learning}, pages
  22234--22259. PMLR, 2022.

\bibitem{wang2020optimizing}
X.~Wang, H.~Pham, P.~Michel, A.~Anastasopoulos, J.~Carbonell, and G.~Neubig.
\newblock Optimizing data usage via differentiable rewards.
\newblock In {\em International Conference on Machine Learning}, pages
  9983--9995. PMLR, 2020.

\bibitem{yang2021provably}
J.~Yang, K.~Ji, and Y.~Liang.
\newblock Provably faster algorithms for bilevel optimization.
\newblock {\em Advances in Neural Information Processing Systems},
  34:13670--13682, 2021.

\end{thebibliography}

\appendix

\section{Proofs}
\subsection{Proof of \autoref{prop:measure_prop}}
Let $\tilde{\theta}$ the global minimizer of the test loss $\theta\mapsto\int\ell(\theta; x)d\nu(x)$.
By definition, for all $\omega$, we have $\int\ell(\theta^*(\omega); x)d\nu(x) \geq \int\ell(\tilde{\theta}; x)d\nu(x)$.
Furthermore, for $\omega^* = d\nu/d\mu$, we have that the inner loss becomes $\int\ell(\theta; x)\omega^*(x)d\mu(x) =  \int\ell(\theta; x)d\nu(x)$, the outer loss. 
Minimizing it leads to $\theta^*(\omega^*) = \tilde{\theta}$, which is, therefore, the global minimizer of the bilevel problem.
\subsection{Proof of \autoref{prop:discretization_ode}}
This is a standard discretization argument from ODE theory.
The gist of the proof is the following result. For a fixed $\alpha, \beta >0$, define
$$
\mathcal{F}(\theta, w, \tau) = \begin{bmatrix}
    -\alpha \nabla G(\theta, w) \\
    \frac1\tau \left(\frac{w \odot \exp(-\tau \beta \Psi(\theta, w))}{\|w \odot \exp(-\tau \beta \Psi(\theta, w))\|_1} - w\right)
\end{bmatrix},
$$
so that the iterations of \autoref{algo:scalablemirrordescent} can be compactly rewritten as $(\theta^{k+1}, w^{k+1}) =(\theta^k, w^k) + \tau  \mathcal{F}(\theta^k, w^k, \tau)$.

We find that 

$$
\lim_{\tau\to0}\mathcal{F}(\theta, w, \tau) = \begin{bmatrix}
    -\alpha \nabla G(\theta, w) \\
    -\beta P(w)\Psi(\theta, w)
\end{bmatrix}, 
$$
hence \autoref{algo:scalablemirrordescent} is indeed a discretization of the ODE~\eqref{eq:bilevel_flow}.
The result follows from classical ODE discretization theory (e.g.~\cite{chicone2006ordinary}).
\subsection{Proof of \autoref{prop:stat}}
The stationary points of the ODE~\eqref{eq:mirror_flow} are the $w\in\Delta_n$ such that $\Phi(w)=0$.
We see that we need to solve the equation $P(w)u=0$ for $u\in \mathbb{R}^n$. 
The $i$-th coordinate of $P(w)u$ is $w_i(u_i - \langle w, u\rangle)$.
Hence, this cancels for all $i$ if and only if $w_i=0$ or $u_i =   \langle w, u\rangle$.
This means that for the $i\in\mathrm{Supp}(w)$, $u_i$ is constant, i.e., $u\mid_{\mathrm{Supp}(w)}$ is proportional to $\mathbb{1}_l$. This gives the adverstized result.

\subsection{Proof of \autoref{prop:stable_flow}}
A stationary point is stable if and only if the operator 
$$
D\Phi(w) = \begin{bmatrix}
    P(\tilde{w})\tilde{J} & (*) \\
    0 & \mathrm{diag}(\phi_i - \langle w, \phi\rangle)
\end{bmatrix}
$$
has positive eigenvalues.
This matrix is block diagonal, hence, the condition is that both $P(\tilde{w})\tilde{J}$ and  $\mathrm{diag}(\phi_i - \langle w, \phi\rangle)$ have positive eigenvalues. 
The condition on $\mathrm{diag}(\phi_i - \langle w, \phi\rangle)$ can simply be rewritten as $\phi_i >  \langle w, \phi\rangle$, finishing the proof.

\subsection{Proof of \autoref{prop:flow_speed}}
The proof is found in~\cite{chicone2006ordinary}, corollary 4.23.

\subsection{Proof of \autoref{prop:constant_frozen_flow}}
In this case, the flow becomes the ODE
$$\dot{w} = -w\odot\phi + w \langle w, \phi\rangle$$
The solution to this ODE is more obvious by looking at the corresponding discrete mirror descent.
In the discrete mirror descente case, it corresponds to the iterates $w^{k+1}=\frac{w^k \exp(-\tau \phi)}{\|w^k \exp(-\tau \phi)\|_1}$, hence a simple recursion shows that $w^k = \frac{w^0 \exp(-\tau k \phi)}{\|w^0 \exp(-\tau k \phi)\|_1}$.
We then infer the solution in continuous time, it is simply
$$
w(t) = \frac{w^0 \odot \exp(-t \phi)}{\|w^0 \odot \exp(-t \phi)\|_1}
$$
We see that is has the adverstized behavior: as $t$ goes to infinity, the i-th coefficient of $\frac{w^0 \odot \exp(-t \phi)}{\|w^0 \odot \exp(-t \phi)\|_1}$ goes to $0$ if $i$ is not in the argmax of $\phi$, and is proportionnal to $w_0$ otherwise.
We also note that this result also holds in the discrete mirror descent case, here the fact that we simplify things to an ODE plays not role in the behavior.
\subsection{Proof of \autoref{prop:set_proba}}
In the case where $l\leq p$, then with probability one, $Z$ is surjective, i.e. the equation $Zx = y$ has a solution $x$ for all $y$.
Therefore, $Z$ is in $\mathcal{I}^p_l$ with probability one.
One the contrary, if $l > p$, then $\mathrm{range}(Z)$ is a random subspace of dimension $p$, hence the probability that $\mathbb{1}_l$ or $\mathbb{0}_l$ are in this space is $0$ -- note that for $\mathbb{0}_l$ we take the subspace range of $Z$ without the $0$ term.

\subsection{Proof of \autoref{prop:sparse_weights_frozen}}
This is simply \autoref{prop:stat} specialized to the field of the form $\phi(w) =\Gamma g(w)$.

Indeed, \autoref{prop:stat} shows that a stationary point must satisfy that $\Gamma g(w)\mid_{\mathrm{Supp}(w)}$ is proportional to $\mathbb{1}_l$, hence $\Gamma\mid_{\mathrm{Supp}(w)} g(w) = \gamma \mathbb{1}_l$ for some $\gamma$. Therefore, $\Gamma\mid_{\mathrm{Supp}(w)}$ is in $\mathcal{I}_l^p$: depending on whether $\gamma=0$ or not, we must have a nonzero solution to $\Gamma\mid_{\mathrm{Supp}(w)}x = \mathbb{0}_l$ or $\Gamma\mid_{\mathrm{Supp}(w)}x = \mathbb{1}_l$.

\subsection{Proof of \autoref{prop:recover_bilevel}}

We will use many times the fact that the simplex $\Delta_n$ is a compact set, hence any continuous function over $\Delta_n$ is bounded. 
We will use the constant $d_\theta = \sup_{w\in\Delta_n} \|\theta^*(w)\|$, which is finite thanks to the compacity of $\Delta_n$.

We first show a very useful lemma: the trajectory of $\theta$ remains bounded.
\begin{lemma}
\label{app:lemma:bounded}
We let $\kappa = \frac{L_G}{\mu}$ the conditioning of $G$.
    If $\|\theta(0)\|\leq R:=\frac{\sqrt{2}(2 +\kappa) + \sqrt{2(2+\kappa)^2 + 8\kappa}}{4} d_\theta$ then for all $t,$ we have $\|\theta(t)\|\leq R$.
\end{lemma}
\begin{proof}
    We let $r(t)=\frac12\|\theta(t)\|^2$. We find
    \begin{align}
        \frac{dr}{dt} &= \langle \dot{\theta}, \theta\rangle\\
        &= -\alpha\langle \nabla G(\theta, w), \theta\rangle\\
        &= -\alpha\left(\langle \nabla G(\theta, w), \theta - \theta^*(w)\rangle  +\langle \nabla G(\theta, w), \theta^*(w)\rangle   \right)\\
        &\leq \alpha\left(-\mu\|\theta - \theta^*(w)\|^2  + \|G(\theta, w)\|\|\theta^*(w)\|\right)\\
        &\leq \alpha(-\mu (\|\theta\|^2 - 2 \|\theta\|d_\theta) + L_Gd_\theta(d_\theta + \|\theta\|))
    \end{align}
where we have used the $\mu-$strong convexity and $L_G$-smoothness of $G$, the inequality $\|\theta - \theta^*\|^2 \geq \|\theta\|^2 - 2\|\theta\|\|\theta^*\|$ and the triangular inequality.
Hence, we get the differential inequality
$$
\frac{dr}{dt} \leq -\alpha(2\mu r^2 - \sqrt{2}d_\theta (2\mu +L_G)r - L_G d_\theta^2):=z(r)
$$
The advertised radius $R$ is exactly the zero of the right-hand side $z(r)$.
By Gronwall's lemma, we have that $r(t)\leq u(t)$ where $u(t)$ is the solution to the ODE $\dot u = z(u)$ with $u(0) = r(0)$. 
Since $R$ is a zero of $z$, the trajectory of $u$ stays in $[0, R]$ following Picard-Lindelof theorem, hence that of $r$ as well, proving the result.

\end{proof}
Note that the radius here does not depend on either $\alpha$ and $\beta$: the dynamics in $\theta$ and $w$ both remain in a compact set that does not depend on the steps.
We will actively use this result throughout the rest of the analysis because it allows us to bound any quantity that depends on $\theta$ and $w$. 
This lemma also immediately implies that the ODE has a solution for all time horizons if $\|\theta(0)\|$ is small enough regardless of $\alpha$ and $\beta$.

To prove the main result, we first control the distance from $\theta(t)$ to $\theta^*(w(t))$. 
We let $\phi_1(t)= \frac12 \|\theta(t) -\theta^*(w(t))\|^2$.
We find
\begin{align}
    \frac{d\phi_1}{dt}(t) &= -\alpha \langle\nabla G(\theta, w), \theta - \theta^*\rangle - \beta \langle P(w)\Psi(\theta, w)\frac{d\theta^*(w)}{d w}, \theta - \theta^*\rangle\\
    &\leq -2\alpha\mu \phi_1(t) + \beta \sqrt{2}K \|\Psi(\theta, w)\|\phi_1^{\frac12},
\end{align}
where $K>0$ is a constant. The first inequality comes from the strong convexity of $G$, and the second one is Cauchy-Schwarz combined with the boundedness of $P(w)$ and $\frac{d\theta^*(w)}{d w}$.
We now use \autoref{app:lemma:bounded} to get a crude bound on $\|\Psi\|$ :$\|\Psi\|\leq a_1$.
Hence, we obtain the differential inequation:
$$
\frac{d\phi_1}{dt}(t)\leq -2\alpha \mu \phi_1 + \sqrt{2}\beta Ka_1 \phi_1^{\frac12}.
$$
This equation is integrated in closed form using the following lemma:
\begin{lemma}
\label{app:lemma:diffeq}
    Let $a, b>0$. The solution to the differential equation $x' = -ax + b\sqrt{x}$ starting from $x(0)\in[0, (b/ a)^2]$ is $x(t) = \left(
    \frac b a(1-\exp(-\frac{at}2)) + \sqrt{x(0)}\exp(-\frac{at}2)\right)^2$.
\end{lemma}
Hence, we get 
\begin{equation}
    \label{app:eq:ineq_distance_theta}
    \phi_1(t)\leq\left(\frac{\beta Ka_1}{\sqrt{2}\alpha \mu}(1-\exp(-\alpha \mu t)) + \sqrt{\phi_1(0)}\exp(-\alpha \mu t)\right)^2\leq \left(\frac{\beta Ka_1}{\sqrt{2}\alpha \mu} + \sqrt{\phi_1(0)}\exp(-\alpha \mu t)\right)^2
\end{equation}
demonstrating the last inequality in \autoref{prop:recover_bilevel}, as the first part of the previous inequality goes to $0$ when $\beta$ goes to $0$.

Next, we consider $\phi_2(t) = \frac12\|w(\frac t \beta) -w^*(t)\|^2$, where $w^*(t)$ is the solution to the ODE $\dot w = -P(w)\nabla h(w)$. We find, where for short $\theta = \theta(\frac t \beta)$, and $w = w(\frac t \beta)$:
\begin{align}
    \frac{d\phi_2}{dt}&=- \langle P(w)\Psi(\theta, w) - P(w^*)\Psi(\theta^*(w^*), w^*), w - w^*\rangle \\
    &\leq \|P(w)\Psi(\theta, w) - P(w^*)\Psi(\theta^*(w^*), w^*)\|\|w - w^*\|\\
    &\leq L_P\left(\|w - w^*\| + \|\theta - \theta^*(w^*)\| \right)\|w - w^*\|,
\end{align}
where $L_P$ is the Lipschitz constant of the map $(w, \theta)\mapsto P(w)\Psi(\theta, w)$, and we have used the Cauchy-Schwarz inequality to control the scalar product.

The term $\|\theta - \theta^*(w^*)\|$ is controled with the triangular inequality by doing
\begin{align}
    \|\theta - \theta^*(w^*)\| &\leq \|\theta - \theta^*(w)\| + \|\theta^*(w^*) - \theta^*(w^*)\|
\end{align}
The first term is $\phi_1(\frac t\beta)$ which is controlled by \autoref{app:eq:ineq_distance_theta}, while the second term is controled by Lipschitzness of $\theta^*$~\cite{ghadimi2018approximation}, yielding  
$$
\|\theta - \theta^*(w^*)\|\leq \frac{\beta Ka_1}{\sqrt{2}\alpha \mu} + \sqrt{\phi_1(0)}\exp(-\frac\alpha\beta \mu t) + L_\theta \|w - w^*\|$$

Overall, we get
$$
\frac{d\phi_2}{dt} \leq L_P(3+L_\theta)\phi_2 + L_P(\phi_1(0)\exp(-2\frac\alpha\beta \mu t) +(\frac{\beta Ka_1}{\sqrt{2}\alpha \mu})^2)
$$
where we have used the inequality $\|w - w^*\|(\frac{\beta Ka_1}{\sqrt{2}\alpha \mu} + \sqrt{\phi_1(0)}\exp(-\frac\alpha\beta \mu t) )\leq \phi_2 \phi_1(0)\exp(-2\frac\alpha\beta \mu t) +(\frac{\beta Ka_1}{\sqrt{2}\alpha \mu})^2$.
Intuitively, this shows that $\phi_2$ is small because as $\beta$ goes to $0$ we get to the inequality $\frac{d\phi_2}{dt}\leq L_P(3+L_\theta)\phi_2$, which implies that $\phi_2=0$ since $\phi_2(0)=0$. 
We use Gronwall's lemma and the fact that $\phi_2(0) = 0$ to prove this intuition. 
This gives

$$
\phi_2(t)\leq \frac{L_P\phi_1(0)}{2\frac{\alpha}{\beta} \mu+ L_P(3+L_\theta)}(\exp(L_P(3+L_\theta)t) - \exp(-2\frac\alpha\beta \mu t)) + \frac{(\beta Ka_1)^2}{2\alpha^2\mu^2( L_P(3+L_\theta))}(\exp( L_P(3+L_\theta)t) - 1)
$$
which indeed goes to $0$ as $\beta$ goes to $0$.

\subsection{Proof of \autoref{prop:sparse_dynamics}}
We need to be able to control the speed at which the mirror flow converges to the solution $\Omega(\theta, w^0),$ to do so we posit a Lyapunov inequality:
\begin{assumption}
    There exists $\nu>0$ such that for all $w, \theta$, it holds
    $$
    \langle P(w)\Psi(\theta, w), w - \Omega(\theta, w^0)\rangle \geq \nu \|w - \Omega(\theta, w^0)\|^2$$
\end{assumption}
This assumption implies, in particular, that the flow $\dot w = - P(w)\Psi(\theta, w)$ with fixed $\theta$ goes to $\Omega(\theta, w^0)$ at an exponential speed.
We also assume that $\Omega$ is $L_\Omega$- Lipschitz.

The analysis is then extremely similar to the previous one: we let $\phi_1(t)= \frac12\|w(t) - \Omega(\theta, w^0)\|^2$ and find
\begin{align}
    \frac{d\phi_1}{dt} &= -\beta  \langle P(w)\Psi(\theta, w), w - \Omega(\theta, w^0)\rangle + \alpha \langle \frac{d\Omega}{d\theta}\nabla G(\theta, w), w - \Omega(\theta, w^0)\rangle \\
    &\leq -2\beta \nu\phi_1(t) + \alpha K \sqrt{\phi_1(t)} 
\end{align}
where $K$ is an upper bound on $\|\frac{d\Omega}{d\theta}\nabla G(\theta, w)\|$.
Therefore, \autoref{app:lemma:diffeq} gives
$$
\phi_1(t)\leq \left(\frac{\alpha K}{2\beta\nu} + \sqrt{\phi_1(0)}\exp(-\beta\nu t)\right)^2
$$

Next, we let $\phi_2(t) = \frac12\|\theta(\frac t\alpha) - \theta^*(t)\|^2$ with $\theta^*(t)$ the solution to the ODE $\dot \theta = - \nabla G(\theta, \Omega(\theta, w^0))$.
We get

\begin{align}
    \frac{d\phi_2}{dt}&=- \langle \nabla G(\theta, w) - \nabla G(\theta^*, \Omega(\theta^*, w^0)), \theta - \theta^*\rangle \\
    &\leq \|\nabla G(\theta, w) - \nabla G(\theta^*, \Omega(\theta^*, w^0)) \|\theta - \theta^*\|\\
    &\leq L_G\left(\|\theta- \theta^*\| + \|w - \Omega(\theta^*, w^0)\| \right)\|\theta - \theta^*\|,
\end{align}
The term $\|w - \Omega(\theta^*, w^0)\|$ is controlled by the triangular inequality:
$$
\|w - \Omega(\theta^*, w^0)\| \leq \|w - \Omega(\theta, w^0)\| + \|\Omega(\theta, w^0) - \Omega(\theta^*, w^0)\|\leq \phi_1(\frac t \alpha) + L_{\Omega} \|\theta - \theta^*\|
$$
and the rest of the proof follows in the same way as that of \autoref{prop:recover_bilevel}.
\end{document}